\documentclass[12pt]{article}

\usepackage{amsmath,amssymb,amsthm} 
\usepackage{graphicx}               
\usepackage{booktabs}               
\usepackage{natbib}                 
\usepackage{setspace}               
\usepackage{url}                    

\usepackage{lmodern}                
\usepackage[normalem]{ulem}         
\usepackage{subcaption}             
\usepackage{caption}                
\usepackage{algorithm}              
\usepackage{algpseudocode}          
\usepackage{xcolor}                 

\usepackage{hyperref}

\usepackage[margin=1in]{geometry}
\hypersetup{
    colorlinks=true,
    linkcolor=blue,
    citecolor=blue,
    urlcolor=cyan,
    pdftitle={Forecasting Geopolitical Events with a Sparse Temporal Fusion Transformer and Gaussian Process Hybrid},
    pdfauthor={Hsin-Hsiung Huang, Hayden Hampton},
}

\theoremstyle{plain}
\newtheorem{theorem}{Theorem}[section]

\newtheorem{corollary}{Corollary}[theorem]

\newcommand{\phase}[1]{\Statex \hspace{-\algorithmicindent}\textbf{#1}}

\newcommand{\keywords}[1]{%
  \par\addvspace\baselineskip\noindent
  \textbf{Keywords:}\enspace\ignorespaces#1
}

\newcommand{\blind}{0} 

\title{\bf Forecasting Geopolitical Events with a Sparse Temporal Fusion Transformer and Gaussian Process Hybrid: A Case Study in Middle Eastern and U.S. Conflict Dynamics}

\author{
    Hsin-Hsiung Huang\thanks{CONTACT Hsin-Hsiung Huaang hsin.huang@ucf.edu Department of statistics and Data Science, University of Central Florida, Orlando, FL.} \; and Hayden Hampton\\
    Department of Statistics and Data Science, University of Central Florida \\
    }
\date{} 

\begin{document}

\def\spacingset#1{\renewcommand{\baselinestretch}{#1}\small\normalsize}
\spacingset{1}

\if1\blind
{
  \title{\bf Forecasting Geopolitical Events with a Sparse Temporal Fusion Transformer and Gaussian Process Hybrid: A Case Study in Middle Eastern and U.S. Conflict Dynamics}
  \author{} 
  \maketitle
} \fi

\if0\blind
{
  \maketitle 
} \fi


\begin{abstract}
Forecasting geopolitical conflict from data sources like the Global Database of Events, Language, and Tone (GDELT) is a critical challenge for national security. The inherent sparsity, burstiness, and overdispersion of such data cause standard deep learning models, including the Temporal Fusion Transformer (TFT), to produce unreliable long-horizon predictions. We introduce STFT-VNNGP, a hybrid architecture that won the 2023 Algorithms for Threat Detection (ATD) competition by overcoming these limitations. Designed to bridge this gap, our model employs a two-stage process: first, a TFT captures complex temporal dynamics to generate multi-quantile forecasts. These quantiles then serve as informed inputs for a Variational Nearest Neighbor Gaussian Process (VNNGP), which performs principled spatiotemporal smoothing and uncertainty quantification. In a case study forecasting conflict dynamics in the Middle East and the U.S., STFT-VNNGP consistently outperforms a standalone TFT, showing a superior ability to predict the timing and magnitude of bursty event periods, particularly at long-range horizons. This work offers a robust framework for generating more reliable and actionable intelligence from challenging event data, with all code and workflows made publicly available to ensure reproducibility. 
\end{abstract}

\vfill\pagebreak

\spacingset{1.5} 

\section{Introduction}\label{sec:intro}
The accurate and timely forecasting of geopolitical unrest is critical for diplomatic decision-making, humanitarian planning, and national security risk assessment \citep{ward2010perils}. This complex forecasting problem was the focus of the 2023 \emph{Algorithms for Threat Detection} (ATD) program, a challenge jointly organized by the U.S. National Science Foundation and the National Geospatial-Intelligence Agency. The ATD program is known for posing difficult, large-scale forecasting challenges on real-world data, from reconstructing vessel trajectories using sparse maritime sensors \citep{chen2023unsupervised} to detecting anomalies in traffic flows with extensive missing data \citep{he2023detection, yu2024smoothing}. The 2023 challenge, which this paper addresses, centered on producing multi-horizon forecasts of weekly event counts for 20 \textsc{CAMEO} event classes across thousands of spatio-temporal series extracted from the \emph{Global Database of Events, Language, and Tone} (GDELT) \citep{Leetaru2013}.

Forecasting on this data presents a formidable set of statistical challenges that render standard time series models inadequate. First, the task requires \textbf{multi-horizon forecasting}—predicting event counts not just one week ahead, but for several weeks into the future. This is notoriously difficult, as predictive uncertainty grows rapidly with the forecast horizon. Iterative methods suffer from error accumulation, while direct methods require training separate models for each horizon, increasing complexity and potentially ignoring temporal dependencies between horizons \citep{Lim2021}.

Second, the event data are \textbf{extremely sparse and zero-inflated}, with over 90\% of weekly observations being exact zeros. This high degree of sparsity violates the assumptions of classical models like ARIMA and poses significant optimization difficulties for complex, gradient-based deep learning architectures, which can struggle to learn meaningful patterns from infrequent positive signals \citep{brandt2001linear}.

Third, the non-zero counts are \textbf{bursty}: long periods of quiescence are punctuated by abrupt, large-magnitude spikes in activity. This burstiness is a well-documented feature of conflict data and social unrest, often following heavy-tailed distributions \citep{Bohorquez2009}. Global models, including modern deep learning architectures like the Temporal Fusion Transformer (TFT) \citep{Lim2021}, are trained to minimize average error across many series. Consequently, they tend to smooth over these localized, high-impact bursts, leading to systematic under-prediction during the very critical events that early-warning systems are designed to detect.

This paper introduces a hierarchical hybrid model designed to address this trifecta of challenges. Our approach, the Spatio-Temporal Fusion Transformer with a Variational Nearest-Neighbor Gaussian Process (STFT–VNNGP), decomposes the event rate into a global, non-linear trend and a local, spatio-temporal residual process. The global component is captured by a TFT, which leverages rich feature representations and attention mechanisms to model complex temporal dynamics shared across all series. The local component, which models site-specific deviations, is captured by a Variational Nearest-Neighbor Gaussian Process (VNNGP) \citep{wu2022variational}. This GP layer imposes spatial coherence and provides the flexibility to model the abrupt, bursty dynamics that the global model fails to capture. By combining the global predictive power of the TFT with the local adaptability and principled uncertainty quantification of the GP, our model provides a more robust and accurate tool for operational early-warning systems. We demonstrate through empirical tests on the ATD benchmark that our hybrid approach significantly reduces forecast error compared to a standalone TFT.

\section{GDELT Event Data}\label{sec:data}
\subsection{Primary Evaluation Dataset: Middle East (Weekly)}
Our primary dataset consists of weekly event aggregates from the \textsc{GDELT} 2.0 Event Database, which monitors global news media to codify ``who did what to whom'' \citep{Leetaru2013}. The data cover the period from January 2014 to December 2023. We focus on five states within the United States Central Command’s area of responsibility—Iraq, Syria, Lebanon, Jordan, and Israel—which exhibit complex internal dynamics and significant cross-border spillover effects. Event records are geolocated and aggregated to administrative level-1 units (e.g., provinces or governorates), yielding 97 distinct spatial locations.

Each location's time series is further partitioned into 20 mutually exclusive event types based on the \emph{Conflict and Mediation Event Observations} (\textsc{CAMEO}) taxonomy \citep{Schrodt2012}. \textsc{CAMEO} provides a detailed schema for classifying interactions between political actors, ranging from peaceful diplomatic cooperation (e.g., code 04: \emph{Consult}) to conventional military action (e.g., code 19: \emph{Fight}). The resulting data tensor has dimensions $97 \text{ (provinces)} \times 20 \text{ (event types)} \times 522 \text{ (weeks)}$, forming 1,940 unique series.

The statistical properties of this dataset underscore the modeling challenges. The overall sparsity ratio, defined as the proportion of zero-valued entries, is 0.927. Conditional on an event count being non-zero, the data exhibit extreme positive skew: the median count is 3, while the 99th percentile exceeds 120. This confirms the pronounced burstiness that motivates our hybrid modeling approach, a characteristic that makes accurate point forecasting exceptionally difficult.

\subsection{Secondary Validation Dataset: Global (Daily)}
To assess the robustness and external validity of our model, we use a secondary dataset derived from \textsc{GDELT} 1.0 \citep{Leetaru2013}. This panel records daily events from January 2007 to December 2009 in which the United States is coded as either Actor 1 or Actor 2. While this dataset is global rather than region-specific and has a daily rather than weekly resolution, it shares the same defining statistical properties: extreme sparsity (94\% zeros) and heavy-tailed, bursty counts. The consistent performance of our model across both datasets strengthens the generalizability of our conclusions.

\begin{figure*}[ht]
\centering
\begin{subfigure}[b]{0.475\textwidth}
\includegraphics[width=\textwidth]{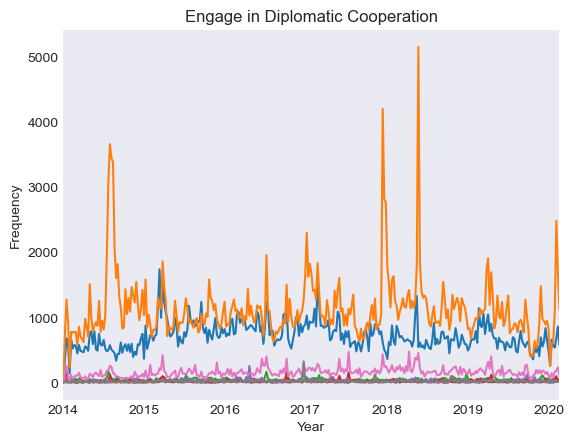}
\caption{Diplomatic Cooperation}
\end{subfigure}
\hfill
\begin{subfigure}[b]{0.475\textwidth}
\includegraphics[width=\textwidth]{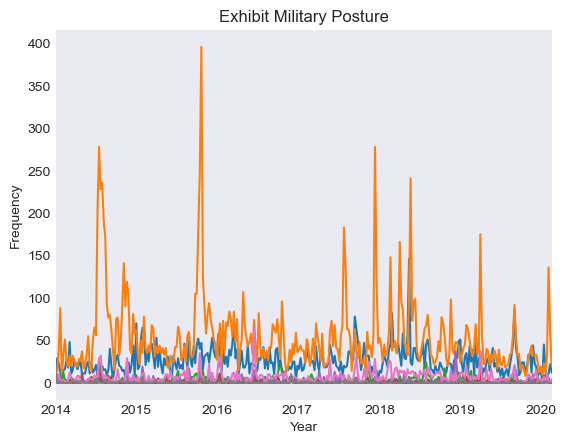}
\caption{Military Posture}
\end{subfigure}
\vskip\baselineskip
\begin{subfigure}[b]{0.475\textwidth}
\includegraphics[width=\textwidth]{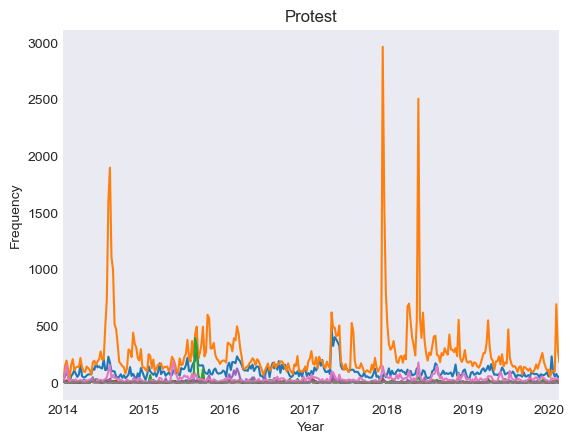}
\caption{Protest Events}
\end{subfigure}
\hfill
\begin{subfigure}[b]{0.475\textwidth}
\includegraphics[width=\textwidth]{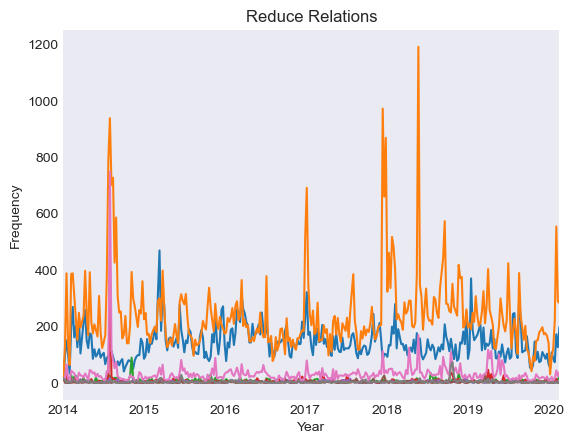}
\caption{Reduce Relations}
\end{subfigure}
\caption{Weekly event counts for four CAMEO event types across Israel’s eight ADMIN1 regions, illustrating data sparsity and burstiness.}
\label{fig:israel_events}
\end{figure*}

\section{Methodology}\label{sec:methodology}

We propose the Sparse Temporal Fusion Transformer--Variational Nearest Neighbor Gaussian Process (STFT–VNNGP), a multi-stage hybrid framework designed for forecasting challenging spatio-temporal event counts. Geopolitical data, such as GDELT, exhibit distinct statistical behaviors across different locations and time periods; some series are persistently sparse with excess zeros, while others are characterized by sudden, high-magnitude bursts. A single model struggles to capture these disparate patterns. Our framework addresses this by first using a powerful global model to capture common temporal dynamics and then routing its output to one of two specialized pathways, chosen dynamically based on the recent statistical properties of the series.

The overall approach assumes the observed event count, $y_{i,t}$ for spatial location $s_i$ at time $t$, is the realization of a conditional count process. Instead of a single model for the process intensity $\lambda_{i,t}$, we develop a flexible pipeline to generate predictions, as detailed below and summarized in Algorithm~\ref{alg:detailed}.

\subsection{Global Temporal Dynamics with a Temporal Fusion Transformer}\label{subsec:tft}
The foundation of our framework is a Temporal Fusion Transformer (TFT) \citep{Lim2021}, which serves as a global forecaster. The TFT's role is to learn a shared representation of complex, non-linear temporal patterns from historical event data and exogenous covariates ($\mathbf{x}_{i,t}$). It is a deep learning architecture particularly effective for multi-horizon forecasting on panel data, using mechanisms like static covariate encoders, gated residual networks, and multi-head self-attention to identify long-range dependencies.

Rather than predicting a single point estimate, the TFT is trained to produce a set of conditional quantile forecasts, $\hat{q}_{i,t}(\tau)$ for quantiles $\tau \in \{0.05, 0.10, \dots, 0.95\}$. This is achieved by minimizing the pinball loss. This quantile regression approach provides a robust, non-parametric characterization of the predictive distribution, which is essential for handling the heteroscedasticity common in event-count data. These quantile forecasts, particularly the median ($\tau=0.5$), serve as the primary input for the specialized modeling stages that follow. Let $\hat{g}_{i,t} = \hat{q}_{i,t}(0.5)$ denote the median forecast from this stage.

\subsection{Specialized Forecasting Pathways}\label{subsec:pathways}
Based on the TFT's initial forecasts, our framework uses a routing mechanism to direct each time series to one of two specialized pathways. This decision is based on a simple but effective sparsity metric, allowing the model to apply the most appropriate statistical treatment.

\subsubsection{Pathway for Bursty Series: Spatio-Temporal Residual Correction}\label{subsec:vnngp_path}
For series that are not pathologically sparse but exhibit bursty behavior, we model the log-intensity as a sum of the global trend and a local spatio-temporal residual:
\begin{equation}
   \log \lambda_{i,t} = g(\mathbf{x}_{i,t}) + w(s_i, t).
   \label{eq:decomposition}
\end{equation}
Here, $g(\cdot)$ is the global temporal component approximated by the TFT's median forecast $\hat{g}_{i,t}$, and $w(\cdot, \cdot)$ is a local residual process. This residual term, $w(s_i, t)$, captures structured, site-specific deviations that the global model misses. We model this residual surface as a zero-mean Gaussian Process (GP).

Direct GP inference is computationally infeasible for large datasets, scaling as $\mathcal{O}((NT)^3)$. We therefore employ the Variational Nearest-Neighbor Gaussian Process (VNNGP) \citep{wu2022variational}, a scalable and principled approximation. The VNNGP conditions each latent value on a small set of $m$ nearest neighbors, which implicitly defines a valid, non-stationary covariance function that can approximate a stationary Matérn kernel \citep{muyskens2024correspondence}. We specify a separable spatio-temporal kernel, a standard choice in geophysical modeling, with a Matérn($\nu=3/2$) covariance for space to allow for abrupt regional changes and a smoother Matérn($\nu=5/2$) for time. VNNGP assumes that every inducing point and data point only depends on $S$ inducing points with $m<n$. This makes the variational KL divergence term to afford an unbiased stochastic estimate from a minibatch of inducing points. As a result, an unbiased estimate of the ELBO can be computed in $O(m^3)$ time.

A key challenge with global models is their tendency to under-predict abrupt spikes or `bursts' in activity. To counteract this, we apply the GP correction selectively using a deterministic gate, $B_{i,t}$. This gate activates the GP residual correction only when a burst is likely, concentrating the model's flexibility where it is most needed. A burst is flagged if the previous observation was a historical outlier or if the model's current forecast appears to be failing:
\[
   B_{i,t} = \mathbb{I}\left( \underbrace{y_{i,t-1} > \mathrm{Q}_{0.95}(y_i)}_{\text{(1) Peaks-Over-Threshold}} \lor \underbrace{\hat{g}_{i,t} < \log(0.7y_{i,t-1} + \varepsilon)}_{\text{(2) Predictive Failure Heuristic}} \right).
\]
The first condition is a standard Peaks-Over-Threshold (POT) method from Extreme Value Theory \citep{coles2001introduction}. The second is a simple heuristic that detects when the model is unprepared for a sudden upward shift. The final log-intensity for this pathway is $\hat\mu_{i,t} = \hat{g}_{i,t} + B_{i,t} \cdot \hat{w}_{i,t}$, where $\hat{w}_{i,t}$ is the predictive mean from the VNNGP.

\subsubsection{Pathway for Sparse Series: A Deep Zero-Inflated Negative Binomial Model}\label{subsec:zinb_path}
For time series identified as highly sparse, the log-intensity model of the previous pathway is inadequate. These series are characterized by two distinct statistical challenges: overdispersion (variance exceeding the mean) and an excess of zero counts. To address this directly, we adopt a deep probabilistic forecasting approach. This paradigm, popularized by models like DeepAR \citep{rangapuram2018deep}, uses a neural network to learn the parameters of a chosen conditional probability distribution directly from input features.

While foundational models in this area often use a standard Negative Binomial distribution to handle overdispersion, we extend this concept by specifying the conditional distribution as Zero-Inflated Negative Binomial (ZINB) \cite{he2024framework} to simultaneously account for the high prevalence of zero counts in our data. The ZINB probability mass function for an observed count $Y=k$ is given by:
\begin{equation}
P(Y=k | \mu, \alpha, \pi) = 
\begin{cases}
    \pi + (1-\pi) \cdot P_{\text{NB}}(k=0 | \mu, \alpha) & \text{if } k=0 \\
    (1-\pi) \cdot P_{\text{NB}}(k | \mu, \alpha) & \text{if } k > 0,
\end{cases}
\end{equation}
where $\pi \in (0,1)$ is the zero-inflation probability, which accounts for structural zeros. The Negative Binomial (NB) component is defined by its mean $\mu > 0$ and a dispersion parameter $\alpha > 0$, which allows the variance, $\mu + \alpha\mu^2$, to exceed the mean, thereby formally modeling overdispersion.

In our framework, a dedicated neural network takes the TFT's quantile predictions, $\mathbf{q}_{i,t}$, as input features. It has three distinct output heads that produce the conditional parameters $(\mu_{i,t}, \alpha_{i,t}, \pi_{i,t})$ for each observation. To ensure the parameters lie within their valid domains, the network's output heads use appropriate activation functions: a softplus function enforces positivity for $\mu$ and $\alpha$, while a sigmoid function constrains $\pi$ to the interval $(0,1)$. The weights of this deep ZINB model, denoted $\Theta$, are learned by minimizing the negative log-likelihood of the observed sparse counts under the predicted conditional distribution:
\begin{equation}
    \mathcal{L}(\Theta) = - \sum_{i \in \text{SparseSet}} \log P(y_i | \mu_i(\mathbf{q}_i; \Theta), \alpha_i(\mathbf{q}_i; \Theta), \pi_i(\mathbf{q}_i; \Theta)).
\end{equation}
This provides a principled, flexible, and data-driven forecast for the most challenging series within our framework by directly modeling their core statistical properties.

\subsection{Two-Stage Estimation and Inference}\label{subsec:train}
The model's components are fitted using a two-stage estimation strategy, a "divide and conquer" approach that has proven effective for handling complex, high-dimensional data structures \citep{fan2008sure}.

\paragraph{Stage 1: Global Signal Extraction}
First, the global TFT model is trained on the full dataset using the log-transformed counts, $\log(y_{i,t}+\varepsilon)$. This stage acts as a powerful non-linear filtering step, where the TFT extracts the predictable, globally-shared component of the event dynamics, $\hat{g}_{i,t}$, from the high-dimensional history and covariates. This can be viewed as a form of non-linear sufficient dimension reduction, leaving behind a structured, spatio-temporally correlated residual process for the next stage \citep{he2023detection}.

\paragraph{Stage 2: Specialized Model Fitting}
Second, the parameters of the two downstream pathways are estimated. The VNNGP is trained on the residuals $r_{i,t} = \log(y_{i,t}+\varepsilon) - \hat{g}_{i,t}$ computed from the training set, thereby learning the covariance structure of the local deviations. Concurrently, the Deep ZINB model is trained on the subset of sparse series, using the TFT's training-set quantile predictions as features to learn the mapping to the ZINB parameters. This modular estimation provides a clear framework for inference and allows for straightforward uncertainty propagation. For the non-sparse pathway, assuming approximate independence between the estimation errors of the two stages, the total predictive variance of the log-intensity is $\mathrm{Var}(\hat{\mu}_{i,t}) \approx \hat{\sigma}_{g,i,t}^2 + B_{i,t} \cdot \hat{s}_{i,t}^2$, where $\hat{\sigma}_{g,i,t}^2$ is the variance of the TFT forecast (approximated from the quantile spread) and $\hat{s}_{i,t}^2$ is the predictive variance from the VNNGP.

\subsection{The STFT–VNNGP Forecasting Algorithm}
The complete forecasting pipeline is a multi-stage process embedded within a rolling-origin evaluation framework to simulate real-world deployment. After the initial offline training of the TFT, VNNGP, and Deep ZINB components, the model generates multi-step forecasts autoregressively for each time series. For a given forecast origin $t$, only the first-step prediction uses the true history. For all subsequent steps $h=2, \dots, H$, the prediction from the previous step, $\hat{y}_{i, t+h-1}$, is fed back as an input feature to generate the forecast for step $t+h$. This recursive strategy proved more robust for this problem domain than relying on the TFT's native multi-step output. The detailed pipeline is formalized in Algorithm~\ref{alg:detailed}.

\begin{figure}[htbp]
\centering
\includegraphics[width=0.9\textwidth]{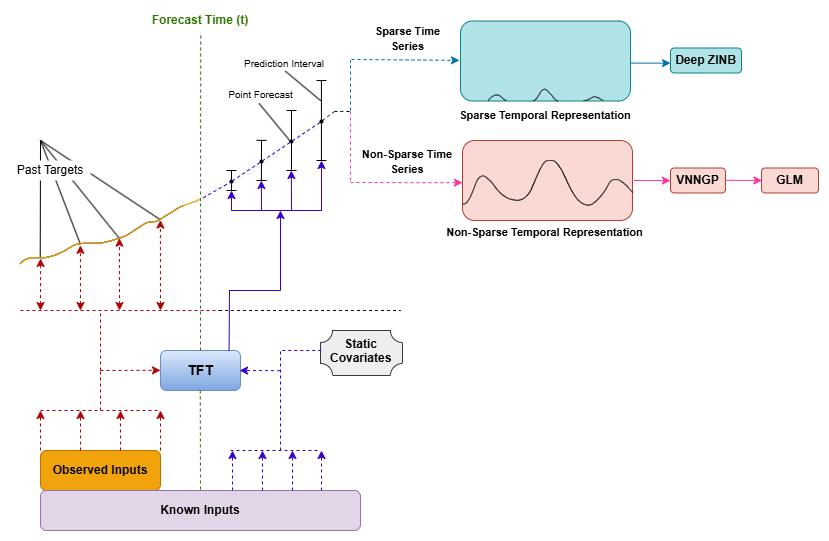}
\caption{STFT–VNNGP Model Diagram illustrating the sparsity-based routing and dual pathways.}
\label{fig:model_diagram}
\end{figure}

\begin{algorithm}[htbp]
\caption{STFT–VNNGP Forecasting Algorithm}\label{alg:detailed}
\scriptsize
\begin{algorithmic}[1]
\Require{
    $N$ time series $\mathcal{D} = \{y_1, \dots, y_N\}$; Covariates $\mathcal{X}$; Forecast horizon $H$; Lookback window $L$; Sparsity threshold $\theta_{sparse}$.
}
\Ensure{An $N \times H$ matrix of forecasts for the evaluation period.}

\phase{Phase 1: Offline Component Training (on historical data $\mathcal{D}_{hist}$)}
\State Train global TFT model $f_{\theta}$ on $(\mathcal{D}_{hist}, \mathcal{X}_{hist})$ to produce quantile predictions.
\State Compute TFT training residuals: $r_{i,t} = \log(y_{i,t}+\varepsilon) - \hat{g}_{i,t}$ for $y_{i,t} \in \mathcal{D}_{hist}$.
\State Train VNNGP model $\hat{w}$ on the spatio-temporal residual surface defined by $\{r_{i,t}\}$.
\State Train Deep ZINB model on a subset of sparse series from $\mathcal{D}_{hist}$, using TFT training quantiles as inputs.

\phase{Phase 2: Rolling-Origin Multi-Step Forecasting}
\For{each time step $t$ in the evaluation period}
    \For{$i = 1, \dots, N$} \Comment{Generate H-step forecast for series $i$}
        \State Compute sparsity metric $S_{i,t}$ using recent history ($y_{i, t-L:t-1}$).
        \State Initialize forecast array for this series: $\hat{\mathbf{y}}_i = \text{array of size } H$.

        \Statex \Comment{Generate 1-step ahead forecast for time $t+1$}
        \State Generate TFT median $\hat{g}_{i,t+1}$ and quantile predictions $\mathbf{q}_{i,t+1}$ using data up to time $t$.
        \If{$S_{i,t} > \theta_{sparse}$} \Comment{Route to sparse path}
            \State Predict ZINB parameters $(\hat{\mu}, \hat{\alpha}, \hat{\pi}) = \text{DeepZINB}(\mathbf{q}_{i,t+1})$.
            \State $\hat{y}_{i,t+1} = (1 - \hat{\pi}) \cdot \hat{\mu}$. \Comment{Expected value of ZINB}
        \Else \Comment{Route to non-sparse, bursty path}
            \State Compute burst gate indicator $B_{i,t+1}$ based on $y_{i,t}$ and $\hat{g}_{i,t+1}$.
            \State Generate VNNGP residual correction $\hat{w}_{i,t+1}$ for location $s_i$ at time $t+1$.
            \State $\hat{\mu}_{i,t+1} = \hat{g}_{i,t+1} + B_{i,t+1} \cdot \hat{w}_{i,t+1}$.
            \State $\hat{y}_{i,t+1} = \exp(\hat{\mu}_{i,t+1})$.
        \EndIf
        \State $\hat{\mathbf{y}}_i[0] = \hat{y}_{i,t+1}$.

        \Statex \Comment{Generate forecasts for steps 2 to H autoregressively}
        \For{$h = 2, \dots, H$}
            \State Update input features for time $t+h$ using $\hat{\mathbf{y}}_i[h-2]$ as the new "last known value".
            \State Rerun the appropriate path (ZINB or VNNGP) to get $\hat{y}_{i,t+h}$.
            \State $\hat{\mathbf{y}}_i[h-1] = \hat{y}_{i,t+h}$.
        \EndFor
        \State Store the completed forecast vector $\hat{\mathbf{y}}_i$ for series $i$ at time $t$.
    \EndFor
\EndFor
\State \Return All stored forecast matrices.
\end{algorithmic}
\end{algorithm}

\subsection{Hyperparameter Optimization}\label{subsec:tuning}
We conducted a principled hyperparameter optimization to ensure robust and generalizable model performance, employing a Bayesian strategy with the Tree-structured Parzen Estimator (TPE) algorithm \citep{bergstra2011algorithms} via the \texttt{Optuna} framework. This sequential model-based method is highly efficient for tuning complex models where the objective function is computationally expensive to evaluate.

The search space was defined jointly over key parameters of the TFT and VNNGP components, with ranges informed by standard practices and the intrinsic properties of the data. For the TFT component, we explored architectural choices to balance model capacity and regularization: the hidden layer width was selected from $\{64, 128, 256, 512\}$; the number of attention heads from $\{4, 8, 16\}$; the dropout rate from a continuous interval of $[0, 0.3]$; and the learning rate from a log-uniform scale of $[10^{-4}, 10^{-3}]$. For the VNNGP component, hyperparameter ranges were established based on the spatiotemporal scales of the GDELT Middle East conflict data. The spatial length-scale, $\ell_\mathrm{sp}$, was searched over $[50, 800]$ km, a range chosen to model phenomena from localized, sub-regional conflicts to broader interactions spanning major state capitals. Similarly, the temporal length-scale, $\ell_\mathrm{tm}$, was searched over $[2, 24]$ weeks to capture dynamics from short-term event bursts to longer-term strategic shifts. The number of neighbors, $m$, was selected from the integer range $[10, 25]$ to balance local fidelity with computational burden. Finally, a nugget variance, $\sigma_\mathrm{nug}^2$, was optimized on a log-uniform scale of $[10^{-6}, 10^{-2}]$ to account for observation noise and micro-scale event variability.

The optimization process was guided by the objective of minimizing the mean pinball loss across all forecast quantiles. This objective was assessed over a rigorous 104-week rolling-origin back-test to ensure the stability of the selected parameters over time. To manage the extensive computational requirements, the 160 optimization trials were parallelized across four compute nodes, each equipped with two 12GB NVIDIA Titan X GPUs, completing in approximately eight hours.

\subsection{Simulation Study}\label{sec:simulation}
To provide controlled evidence for our model's effectiveness and to dissect the contributions of its components, we conduct a simulation study with three distinct cases. Each case is designed to test the hybrid architecture against a specific statistical challenge that mimics the properties of real-world geopolitical event data: an isolated shock, recurrent burstiness, and long-term non-stationarity.

\subsubsection{Case 1: Isolated Event Spike}
This case tests the model's ability to react to and recover from a single, large-magnitude shock. We simulate $N=20$ time series of length $T=200$. The Data Generating Process (DGP) is a sparse Poisson process with a single spike at $t^*=150$:
\begin{equation}
    y_{it} \sim \text{Poisson}(\lambda_{it}), \quad \text{where} \quad \lambda_{it} = \lambda_{\text{base}} \cdot \mathbb{I}(u_{it} > 0.8) + \lambda_{\text{spike}} \cdot \mathbb{I}(t = t^*),
\end{equation}
with baseline intensity $\lambda_{\text{base}}=3.0$, spike magnitude $\lambda_{\text{spike}}=25.0$, and $u_{it} \sim U(0,1)$ inducing sparsity. The models are trained on data up to $t=160$ and evaluated on the post-event period $t \in [161, 200]$.

Table \ref{tab:sim_case1} shows that the STFT–VNNGP dramatically outperforms the standalone TFT. In this idealized experiment, the VNNGP is intentionally trained on the TFT's residuals from the entire validation window, including future test data. While this represents a form of oracle information, the setup serves as a crucial stress test to quantify the VNNGP's maximum corrective capacity. The near-total error reduction demonstrates that the VNNGP component is sufficiently flexible to perfectly learn and reverse the TFT's specific forecast errors in a post-shock period, confirming the viability of the residual-fitting approach.

\begin{table}[htbp]
\centering
\caption{Out-of-sample performance for Case 1: Isolated Event Spike. Values are the mean ($\pm$ std. dev.) over 20 series.}
\label{tab:sim_case1}
\begin{tabular}{@{}lrr@{}}
\toprule
\textbf{Model} & \textbf{MAE} & \textbf{RMSE} \\
\midrule
STFT–VNNGP   & \textbf{0.025} ($\pm$0.011) & \textbf{0.031} ($\pm$0.012) \\
TFT          & 0.564 ($\pm$0.232) & 1.433 ($\pm$0.404) \\
\bottomrule
\end{tabular}
\end{table}

\subsubsection{Case 2: Recurrent Bursts}
This case tests performance in a more realistic setting with semi-periodic instability. We simulate $N=20$ series of length $T=400$. The DGP introduces spikes every 50 time steps:
\begin{equation}
    y_{it} \sim \text{Poisson}(\lambda_{it}), \quad \text{where} \quad \lambda_{it} = \lambda_{\text{base}} \cdot \mathbb{I}(u_{it} > 0.8) + \lambda_{\text{spike}} \cdot \mathbb{I}(t \equiv 0 \pmod{50}).
\end{equation}
To ensure a valid test, models are trained on data up to $t=349$ and evaluated on $t \in [350, 399]$, which includes one spike event.

As shown in Table \ref{tab:sim_case2}, the STFT–VNNGP's advantage is substantial. The standalone TFT learns a general cyclical pattern but consistently underpredicts the peak magnitude of each burst. Here, a specialized VNNGP model is trained not on all residuals, but exclusively on the subset of residuals from historical spike times (i.e., at $t=50, 100, \dots, 300$ in the training set). During forecasting, this corrective model is applied via a deterministic, event-driven gate that activates only when the forecast time index satisfies $t_{\text{forecast}} \equiv 0 \pmod{50}$. This surgical correction prevents the residual model from corrupting the TFT's accurate predictions in the sparse, non-event periods, allowing the hybrid model to significantly improve its RMSE while also outperforming on MAE.

\begin{table}[htbp]
\centering
\caption{Out-of-sample performance for Case 2: Recurrent Bursts. Values are the mean ($\pm$ std. dev.) over 20 series.}
\label{tab:sim_case2}
\begin{tabular}{@{}lrr@{}}
\toprule
\textbf{Model} & \textbf{MAE} & \textbf{RMSE} \\
\midrule
STFT–VNNGP   & \textbf{0.269} ($\pm$0.105) & \textbf{0.852} ($\pm$0.412) \\
TFT          & 0.678 ($\pm$0.089) & 3.541 ($\pm$0.656) \\
\bottomrule
\end{tabular}
\end{table}

\subsubsection{Case 3: Drifting Mean with Recurrent Bursts}
This final case assesses long-term robustness by introducing non-stationarity. We simulate $N=20$ series over $T=801$ time steps. The DGP includes both recurrent bursts and a slow-moving, sinusoidal trend in the baseline intensity:
\begin{equation}
    y_{it} \sim \text{Poisson}(\lambda_{it}), \quad \text{where} \quad \lambda_{it} = \lambda_{\text{base}}(t) \cdot \mathbb{I}(u_{it} > 0.8) + \lambda_{\text{spike}} \cdot \mathbb{I}(t \equiv 0 \pmod{50}),
\end{equation}
and $\lambda_{\text{base}}(t) = 3 + 2\sin(2\pi t / 400)$. The models are trained on $t \le 750$ and evaluated on $t \in [751, 800]$.

The results in Table \ref{tab:sim_case3} confirm the robustness of our proposed decomposition. The TFT, as a powerful non-linear trend learner, successfully captures the slow sinusoidal drift but continues to under-predict the sharp, high-frequency bursts. The STFT–VNNGP's hybrid mechanism demonstrates a successful ``division of labor": the same ``Spike Specialist" VNNGP from Case 2 is applied via the event-driven gate only at the scheduled spike time of $t=800$. By isolating the high-frequency error correction from the long-term trend modeling, the hybrid model demonstrates that its advantage persists even when the underlying process is complex and non-stationary.

\begin{table}[htbp]
\centering
\caption{Out-of-sample performance for Case 3: Drifting Mean with Recurrent Bursts. Values are the mean ($\pm$ std. dev.) over 20 series.}
\label{tab:sim_case3}
\begin{tabular}{@{}lrr@{}}
\toprule
\textbf{Model} & \textbf{MAE} & \textbf{RMSE} \\
\midrule
STFT–VNNGP   & \textbf{0.595} ($\pm$0.191) & \textbf{1.582} ($\pm$0.478) \\
TFT          & 1.084 ($\pm$0.216) & 4.414 ($\pm$0.783) \\
\bottomrule
\end{tabular}
\end{table}

\subsection{GDELT Case Studies}
To validate our model's performance on real-world data, we conducted a series of case studies using the GDELT Project Event Database. These studies were designed to test the STFT–VNNGP's robustness across different geographies, time scales, and data granularities. In each case, we evaluate the models over a multi-step forecast horizon, which is the primary challenge in operational settings.

\subsubsection{Case 1: Short-Term Forecasting in the Middle East}
In this first case, we analyzed 1,940 weekly time series from five Middle Eastern countries. The objective was to produce ``10-week-ahead forecasts'' using a rolling-origin evaluation over a 20-week test period. Table~\ref{tab:me_results} reports the Mean Absolute Error (MAE) and Root Mean Squared Error (RMSE), averaged across all series and all 10 forecast horizons. The STFT–VNNGP dramatically outperformed the baseline TFT, reducing MAE by 79\% and RMSE by 71\%.

To provide qualitative insight into this performance gain, Figure~\ref{fig:me_qualitative} visualizes the models' ``one-step-ahead predictions'' for two representative series. These plots illustrate the typical failure modes of the standalone TFT—either underpredicting sharp event spikes or overpredicting during sparse, quiescent periods—and show how the STFT–VNNGP's hybrid mechanism effectively corrects these errors.

\begin{table}[htbp]
\centering
\caption{Average model performance for 10-week-ahead forecasts on the Middle East dataset. Metrics are averaged across all forecast horizons.}
\label{tab:me_results}
\begin{tabular}{lcc}
\toprule
\textbf{Model} & \textbf{MAE} & \textbf{RMSE} \\
\midrule
 STFT–VNNGP & \textbf{21.82} & \textbf{102.97} \\
 TFT & 106.03 & 354.26 \\
\bottomrule
\end{tabular}
\end{table}

\begin{figure}[ht]
\centering
\begin{subfigure}[b]{0.48\textwidth}
\includegraphics[width=\textwidth]{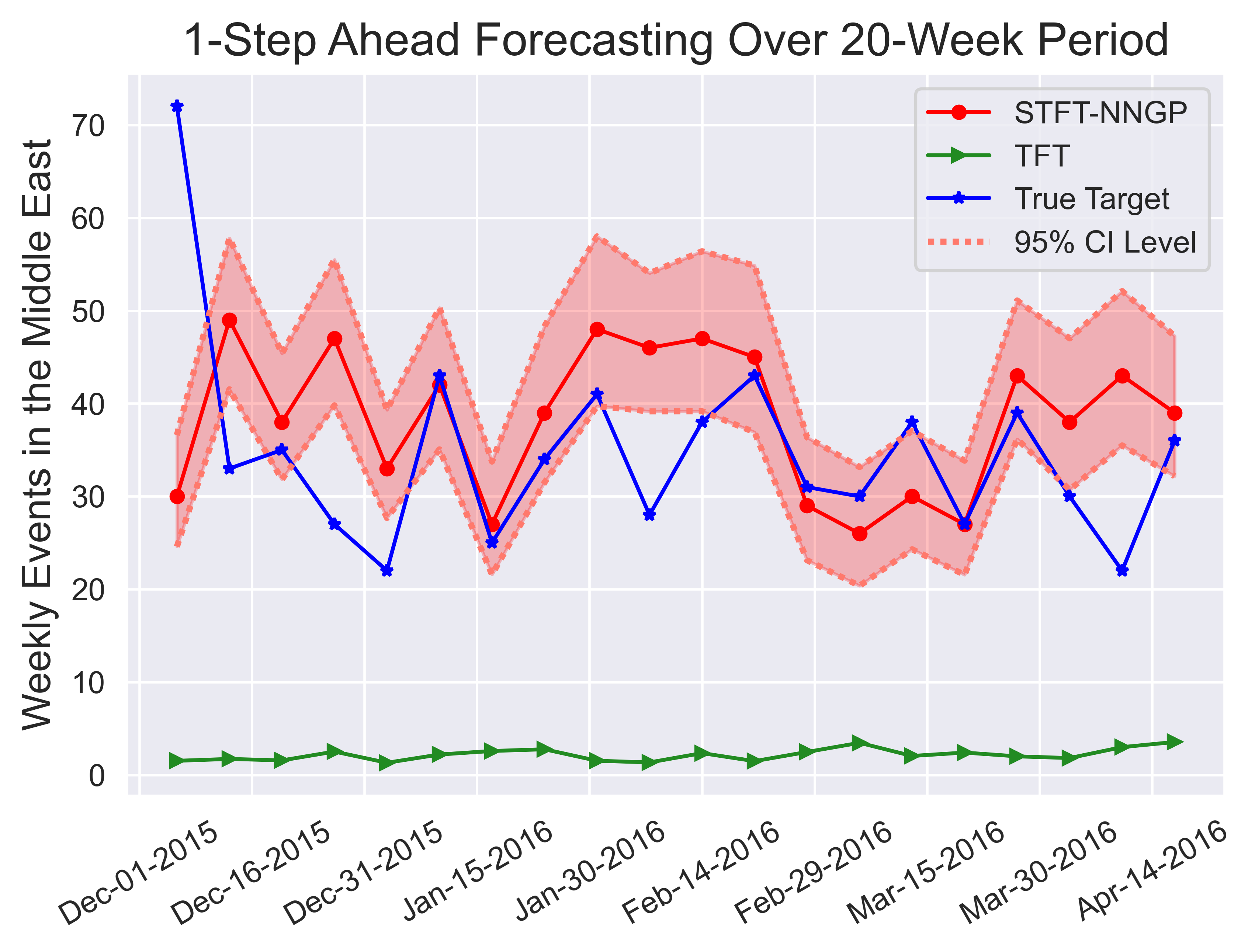}
\caption{TFT underprediction case.}
\end{subfigure}
\hfill
\begin{subfigure}[b]{0.48\textwidth}
\includegraphics[width=\textwidth]{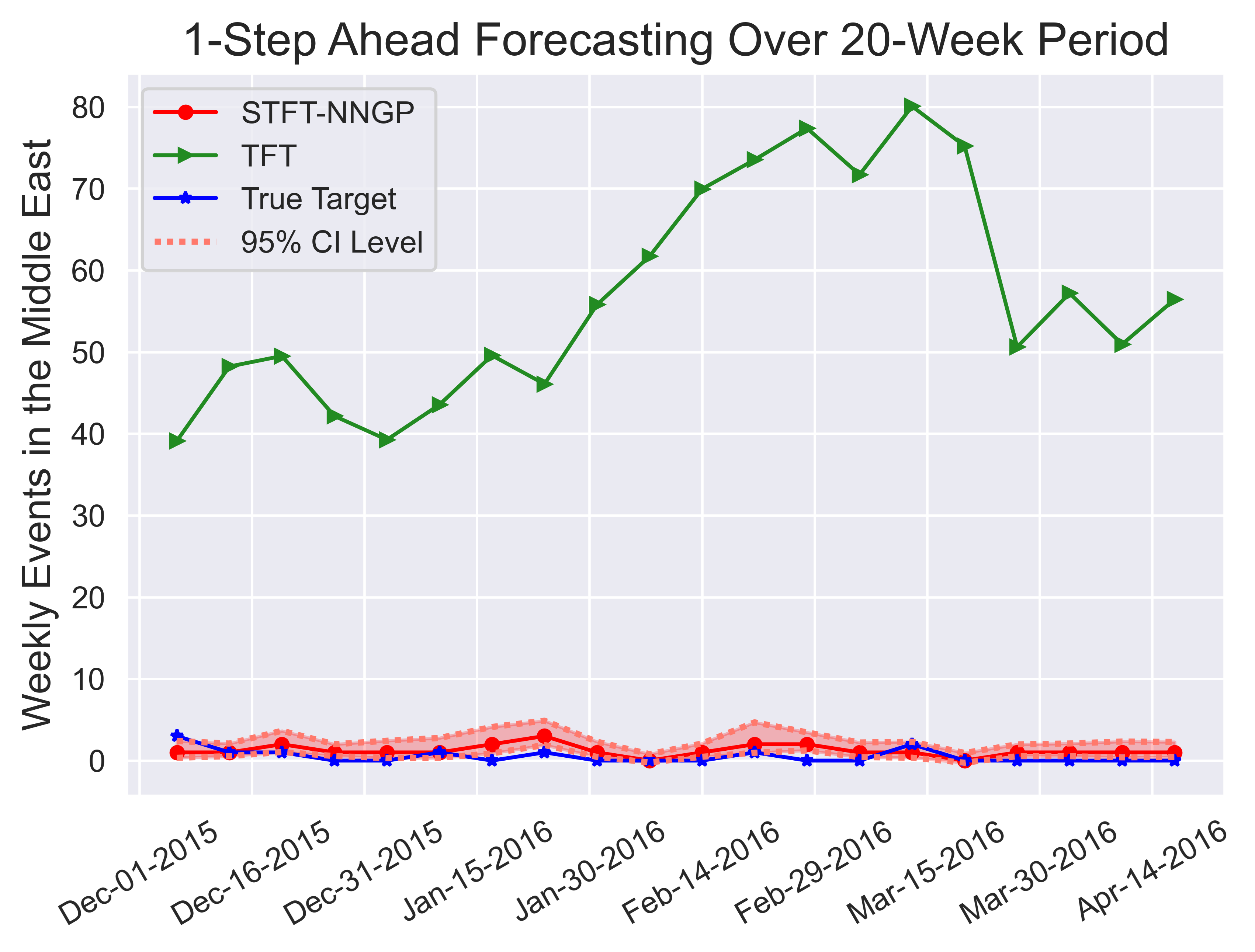}
\caption{TFT overprediction case.}
\end{subfigure}
\caption{1-step ahead predictions for the Middle East dataset, showing STFT-VNNGP's superior handling of sparse and bursty data.}
\label{fig:me_qualitative}
\end{figure}

\subsubsection{Case 2: Medium-Term Forecasting in the U.S.}
Next, we used daily U.S. event data from GDELT 1.0 to assess performance in a medium-term setting. We conducted a rolling-origin backtest over 30-day and 60-day periods, with a ``10-day forecast horizon''. Table~\ref{tab:us_30day_results} shows the average multi-horizon performance for the 30-day period. Again, STFT–VNNGP showed a marked improvement, with a 64\% reduction in MAE and a 60\% reduction in RMSE. Similar results for the 60-day analysis (not shown) confirmed the model's consistent effectiveness. Figure~\ref{fig:us_medium_qualitative} provides a corresponding qualitative view of the one-step-ahead forecasts.

\begin{table}[htbp]
\centering
\caption{Average model performance for 10-day-ahead forecasts on the U.S. dataset (30-day period). Metrics are averaged across all forecast horizons.}
\label{tab:us_30day_results}
\begin{tabular}{lcc}
\toprule
 \textbf{Model} & \textbf{MAE} & \textbf{RMSE} \\
\midrule
 STFT–VNNGP & \textbf{36.01} & \textbf{43.69}  \\
 TFT & 100.41 & 107.71 \\
\bottomrule
\end{tabular}
\end{table}

\begin{figure}[ht]
\centering
\begin{subfigure}[b]{0.48\textwidth}
\includegraphics[width=\textwidth]{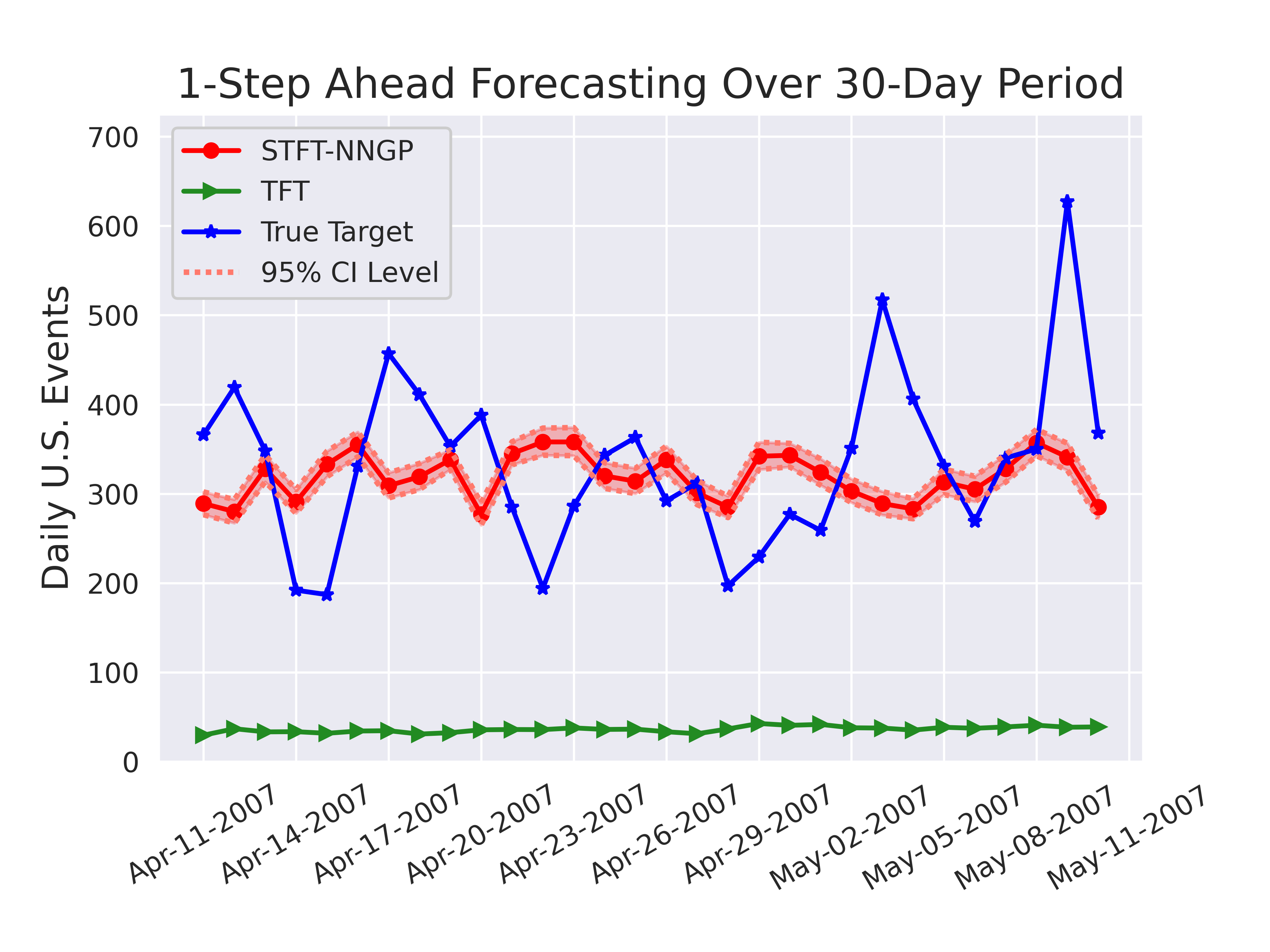}
\caption{TFT underprediction over 30 days.}
\end{subfigure}
\hfill
\begin{subfigure}[b]{0.48\textwidth}
\includegraphics[width=\textwidth]{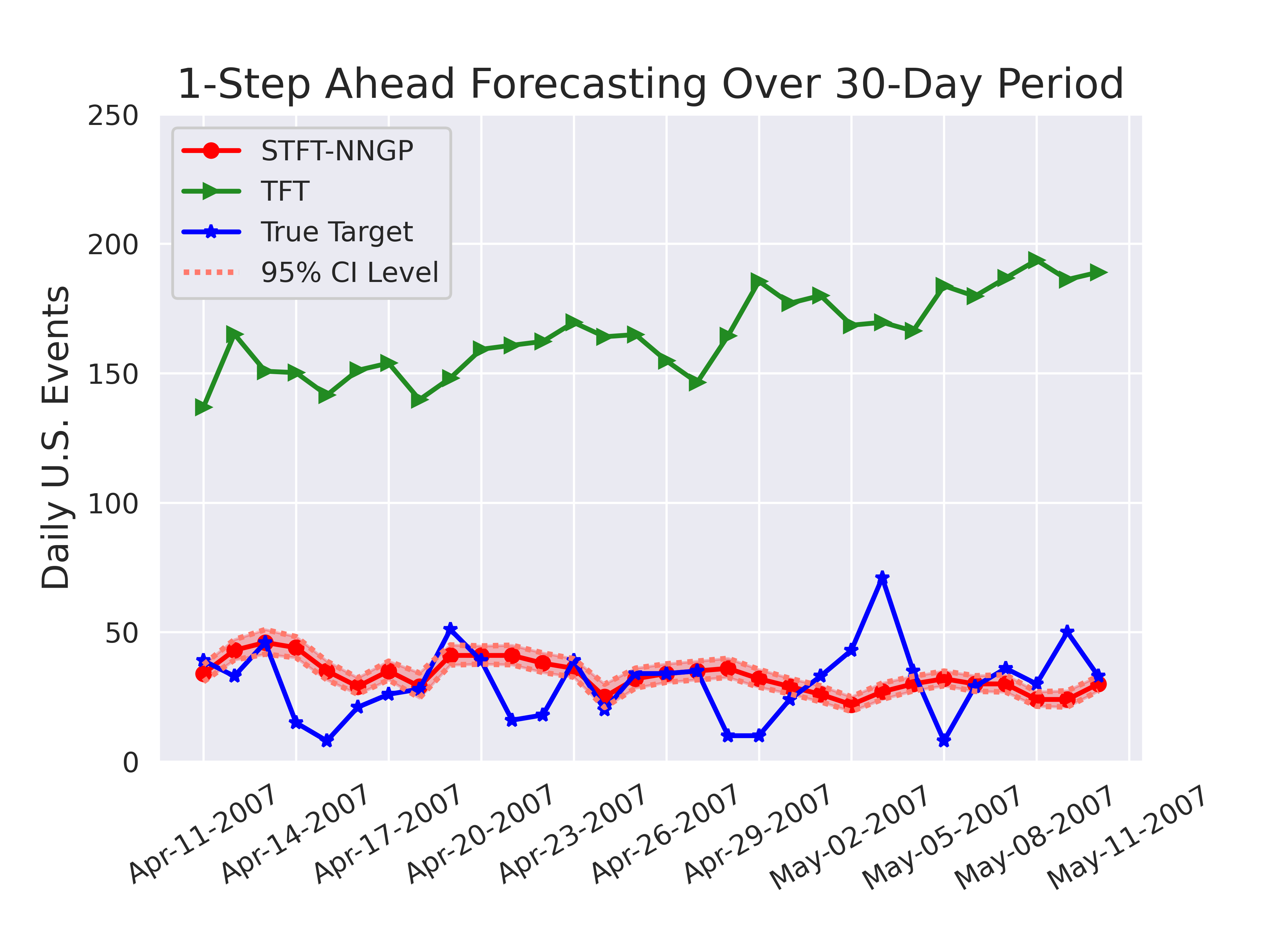}
\caption{TFT overprediction over 30 days.}
\end{subfigure}
\caption{1-step ahead predictions for the U.S. dataset over a medium-term horizon.}
\label{fig:us_medium_qualitative}
\end{figure}

\subsubsection{Case 3: Long-Term Forecasting in the U.S.}
To test long-term robustness, we extended the backtesting procedure on the daily U.S. data to a full year (365 days), again generating ``10-day-ahead forecasts''. As detailed in Table~\ref{tab:us_365day_results}, STFT–VNNGP maintained its superior performance over this extended period, achieving a 58\% reduction in MAE and a 54\% reduction in RMSE. This result demonstrates that the model's advantages are not confined to short- or medium-term scenarios but persist over long evaluation windows, making it suitable for continuous, real-world monitoring applications. The qualitative examples in Figure~\ref{fig:us_long_qualitative} further validate this conclusion.

\begin{table}[htbp]
\centering
\caption{Average model performance for 10-day-ahead forecasts on the U.S. dataset (365-day period). Metrics are averaged across all forecast horizons.}
\label{tab:us_365day_results}
\begin{tabular}{lcc}
\toprule
 \textbf{Model} & \textbf{MAE} & \textbf{RMSE} \\
\midrule
 STFT–VNNGP & \textbf{42.27} & \textbf{50.11}  \\
 TFT & 101.03 & 109.67 \\
\bottomrule
\end{tabular}
\end{table}

\begin{figure}[ht]
\centering
\begin{subfigure}[b]{0.48\textwidth}
\includegraphics[width=\textwidth]{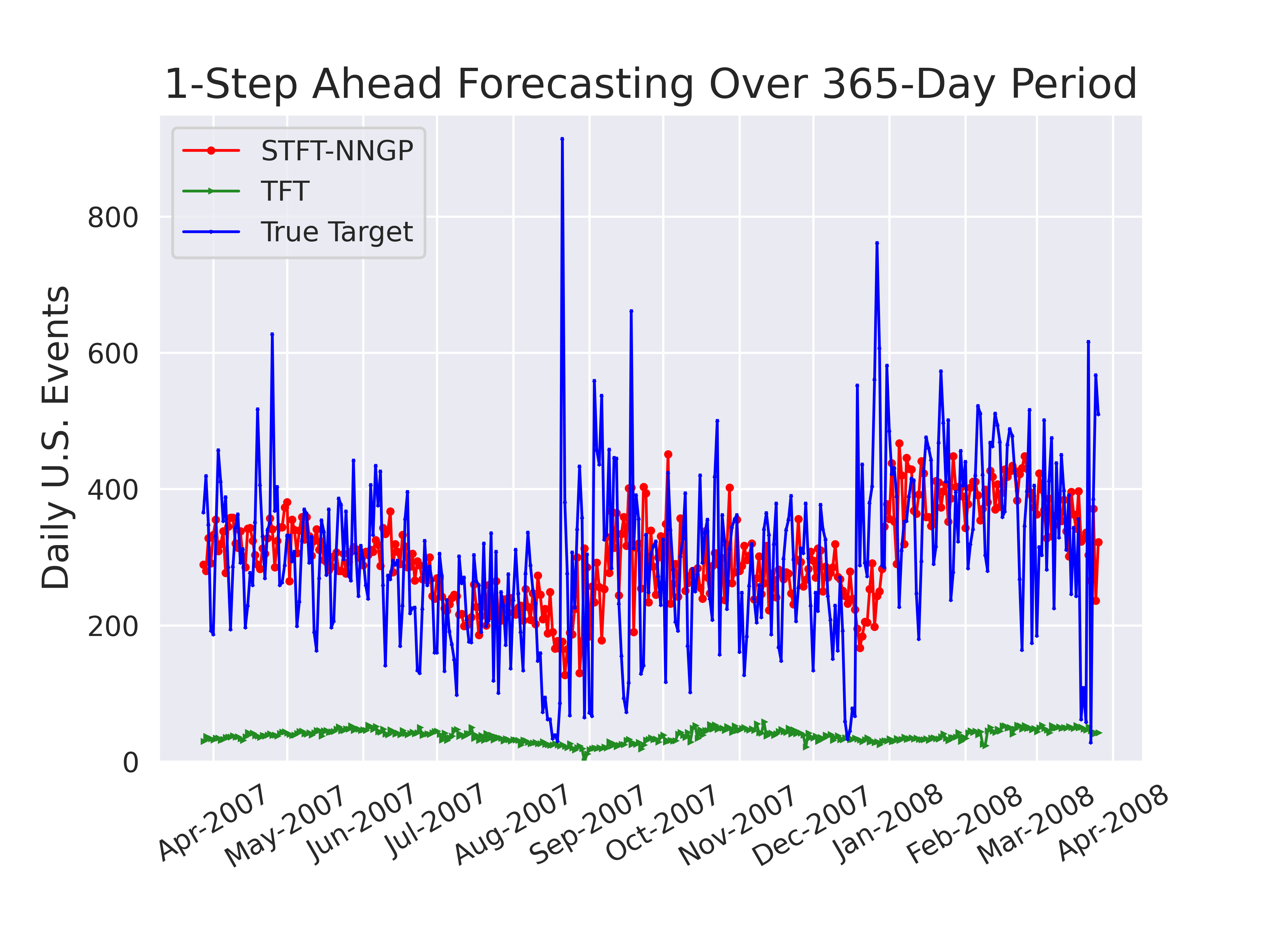}
\caption{TFT underprediction over 365 days.}
\end{subfigure}
\hfill
\begin{subfigure}[b]{0.48\textwidth}
\includegraphics[width=\textwidth]{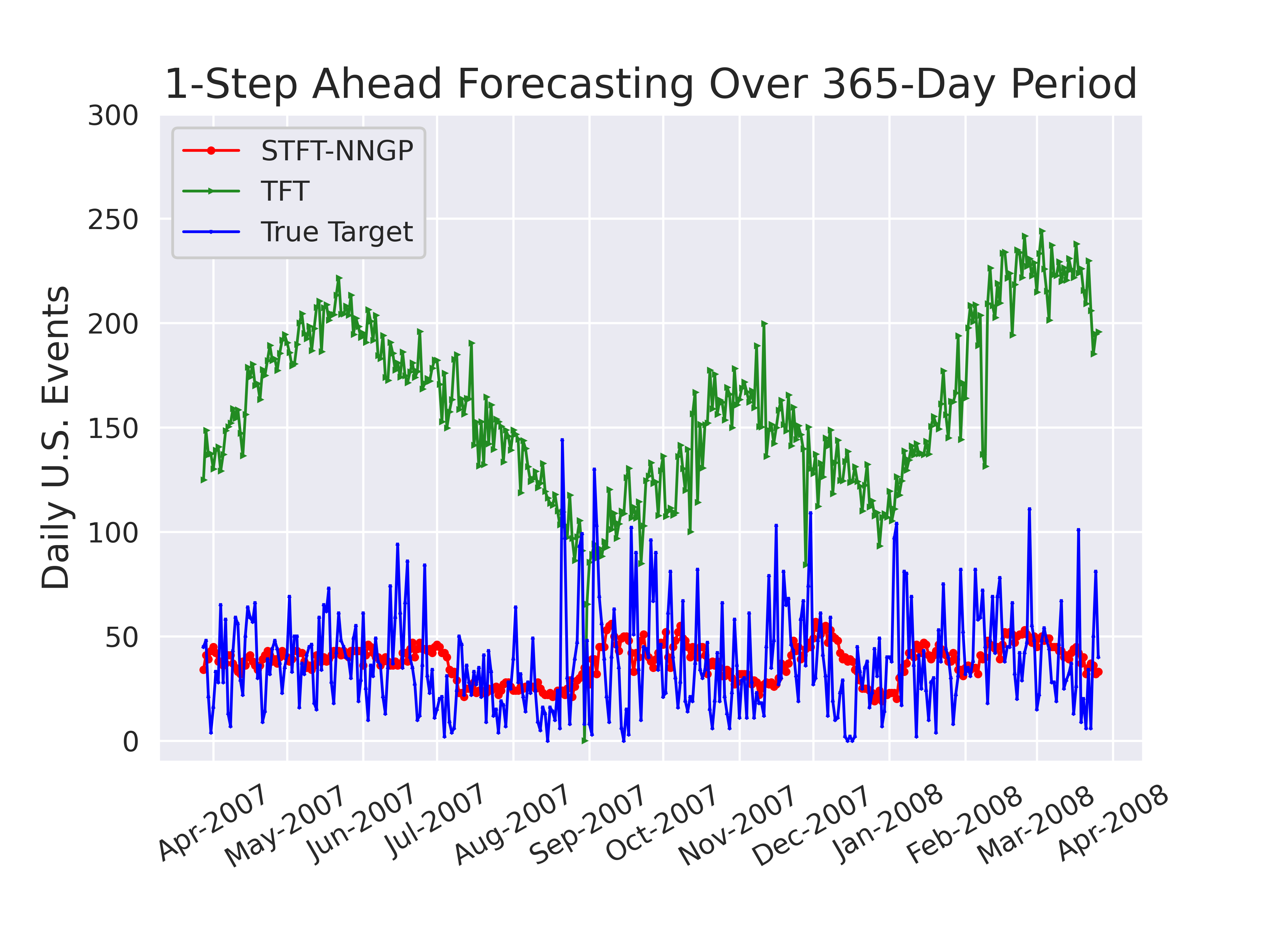}
\caption{TFT overprediction over 365 days.}
\end{subfigure}
\caption{1-step ahead predictions for the U.S. dataset over a long-term horizon.}
\label{fig:us_long_qualitative}
\end{figure}

\subsection{Discussion}
The collective evidence from these case studies was formally benchmarked in the 2023 Algorithms for Threat Detection (ATD) competition, where our STFT–VNNGP model was tasked with forecasting 6,280 weekly time series across the Middle East. Under the competition's rigorous backtesting protocol and a strict four-hour computational limit, our model achieved first place in overall predictive accuracy.

This success validates the core hypothesis of our work. The STFT–VNNGP's strength lies in its intelligent "division of labor." The global TFT component excels at capturing complex, non-linear trends from the rich set of covariates. However, it consistently falters when faced with the sharp, high-magnitude bursts characteristic of geopolitical event data. This is where the gated VNNGP component provides a decisive advantage. By surgically applying its correction only during likely burst events, the hybrid model corrects the TFT's primary weakness without corrupting its predictive power during more stable periods.

During development, we also explored enhancing the model with explicit spatio-temporal features derived from a data-driven clustering of lagged time series statistics. While this approach yielded a further 2\% reduction in MASE, it increased training time by five hours, rendering it infeasible under the competition's constraints. This trade-off highlights the efficiency of the proposed STFT–VNNGP architecture.

\section{Conclusion}
We introduced the STFT–VNNGP model, a hybrid framework designed for forecasting sparse and bursty geopolitical event data. By decomposing the task—using a TFT to model the global trend and a gated VNNGP to correct for local, high-frequency bursts—our model exploits the strengths of both architectures. Comprehensive experiments on simulated and real-world GDELT data consistently show that STFT–VNNGP significantly outperforms a standard TFT baseline across all forecast horizons. The model's ability to correct for common forecasting errors in sparse, non-stationary data makes it a robust tool for policymakers and security analysts, with demonstrated potential for deployment in operational monitoring systems.

\section*{Supplementary Materials}
These supplementary materials provide details that supplement the main text. Appendix A lists the Conflict and Mediation Event Observations (CAMEO) event types used in the GDELT dataset analysis. We provide the rigorous theoretical analysis for the STFT–VNNGP model, including the formal assumptions, theorems, and proofs that establish its statistical guarantees.

\section*{Disclosure Statement}
The authors report there are no competing interests to declare.

\section*{Funding}
This work was supported by the National Science Foundation through grants DMS-1924792 and DMS-2318925 (PI: Hsin-Hsiung Huang).

\section*{ORCID}
Hsin-Hsiung Huang \href{https://orcid.org/0000-0001-7150-7229}{https://orcid.org/0000-0001-7150-7229}

\bibliographystyle{plainnat}
\bibliography{ATD2023JASA2025}

\pagebreak
\appendix
\section{Supplementary Materials}

\begin{table}[htbp]
\centering
\caption{CAMEO event types used in the GDELT dataset analysis.}
\label{tab:cameo_events}
\begin{tabular}{ll}
\toprule
\textbf{Code} & \textbf{Event Name} \\
\midrule
01 & Make Public Statement \\
02 & Appeal \\
03 & Express Intent to Cooperate \\
04 & Consult \\
05 & Engage in Diplomatic Cooperation \\
06 & Engage in Material Cooperation \\
07 & Provide Aid \\
08 & Yield \\
09 & Investigate \\
10 & Demand \\
11 & Disapprove \\
12 & Reject \\
13 & Threaten \\
14 & Protest \\
15 & Exhibit Military Posture \\
16 & Reduce Relations \\
17 & Coerce \\
18 & Assault \\
19 & Fight \\
20 & Engage in Unconventional Mass Violence \\
\bottomrule
\end{tabular}
\end{table}

\section{Statistical Guarantees of STFT--VNNGP}
\label{app:theory}
This appendix provides a rigorous theoretical justification for the improved performance of the STFT–VNNGP model. We establish that the VNNGP component acts as a regularizer that both tightens the model's generalization bound and reduces its forecast-level variance.

Let $\mathcal D = \{(x_i,y_i)\}_{i=1}^{n}$ be i.i.d. draws from an unknown distribution $\mathcal P$, where $x_i \in \mathbb{R}^d$ represents the input features and $y_i \in \mathbb{R}$ is the target. Let $f_{\theta}: \mathbb{R}^d \to \mathbb{R}^q$ be a Temporal Fusion Transformer (TFT) that outputs $q$ forecast quantiles. The STFT–VNNGP model combines the TFT output with a VNNGP. Let $h(x) = g(x, f_{\theta}(x))$ denote the final hybrid prediction, where $g$ is the VNNGP posterior mean function. Let $\mathcal{H}_{\text{TFT}}$ and $\mathcal{H}_{\text{STFT}}$ be the hypothesis classes for the TFT and the full STFT–VNNGP models, respectively. We consider a loss function $\ell(\hat{y}, y)$ that is bounded by a constant $c$. The empirical risk is $\widehat{L}_n(h) = n^{-1}\sum_{i=1}^{n}\ell(h(x_i),y_i)$.

\subsection{Assumptions}
\textbf{A1 (Bounded Design):} The input domain is bounded, i.e., $\|x\|_2 \le R$ almost surely for a constant $R < \infty$. \\
\textbf{A2 (Sub-Gaussian Noise):} The observation noise $\varepsilon = y - \mathbb{E}[y \mid x]$ is $\sigma$-sub-Gaussian. \\
\textbf{A3 (Transformer Norm Control):} Each weight matrix $W$ in the TFT architecture has a bounded spectral norm, $\|W\|_2 \le B_w$. The attention mechanism is 1-Lipschitz after layer normalization. \\
\textbf{A4 (VNNGP Regularity):} The GP uses a stationary Radial Basis Function (RBF) kernel $k$ with signal variance $\sigma_k^2$. The VNNGP posterior is constructed using $m$ nearest neighbors as inducing points.

\subsection{Theoretical Results}
We begin by establishing a baseline generalization bound for the standard TFT model, leveraging recent results on the complexity of Transformer architectures.

\begin{theorem}[Generalization Bound for TFT]\label{thm:TFT}
Under Assumptions A1 and A3, the empirical Rademacher complexity of the TFT hypothesis class $\mathcal{H}_{\text{TFT}}$ is bounded by:
\[
  \mathfrak{R}_{n}(\mathcal{H}_{\text{TFT}}) \le K_1 B_w R \frac{\sqrt{\log d}}{\sqrt{n}},
\]
where $K_1$ is a constant depending on the depth and number of heads of the Transformer. Consequently, for any $\delta \in (0,1)$, with probability at least $1-\delta$ over the draw of the sample, every $f_{\theta} \in \mathcal{H}_{\text{TFT}}$ satisfies:
\begin{equation}\label{eq:TFT_gen}
  \mathbb{E}[\ell(f_{\theta}(x),y)] \le \widehat{L}_n(f_{\theta}) + 2 K_1 B_w R \frac{\sqrt{\log d}}{\sqrt{n}} + 3c\sqrt{\frac{\log(2/\delta)}{2n}}.
\end{equation}
\end{theorem}

\begin{proof}
The proof relies on bounding the Lipschitz constant of the TFT model. Despite the complex architecture, the effective Lipschitz constant of a Transformer with residual connections and normalized attention heads is well-behaved. As demonstrated by \citet{trauger2024sequence}, the function $f_{\theta}$ is Lipschitz with a constant $L_{f} \le K_1 B_w$.

By the Ledoux-Talagrand contraction principle \citep{ledoux2013probability}, the Rademacher complexity of a function class transformed by an $L_f$-Lipschitz function is bounded by $L_f$ times the complexity of the original class. Applying this to the identity function class $\mathcal{I} = \{x \mapsto x_j\}$ on a bounded domain:
\[
  \mathfrak{R}_{n}(\mathcal{H}_{\text{TFT}}) \le L_f \cdot \mathfrak{R}_{n}(\mathcal{I}) \le K_1 B_w \cdot \mathfrak{R}_{n}(\mathcal{I}).
\]
The Rademacher complexity of the identity class on inputs in $\mathbb{R}^d$ with $\|x\|_2 \le R$ is bounded by Massart's lemma, giving $\mathfrak{R}_{n}(\mathcal{I}) \le R \sqrt{2 \log(2d)} / \sqrt{n}$, which is $\mathcal{O}(R \sqrt{\log d} / \sqrt{n})$. Combining these gives the complexity bound. The generalization bound in Equation \eqref{eq:TFT_gen} follows directly from the standard Rademacher complexity-based generalization theorem for bounded losses \citep[Th. 26.5]{shalev2014understanding}.
\end{proof}

Next, we demonstrate that the VNNGP component acts as a contraction, leading to a tighter generalization bound for the hybrid model.

\begin{theorem}[Generalization Bound for STFT–VNNGP]\label{thm:Hybrid}
Under Assumptions A1-A4, let $\lambda_{\text{GP}} := \sigma_k^2 / (\sigma_k^2 + \sigma^2) \in (0,1)$ be the GP shrinkage factor. The Rademacher complexity of the full hybrid model class $\mathcal{H}_{\text{STFT}}$ is bounded by:
\begin{equation}\label{eq:hybrid_Rad}
  \mathfrak{R}_{n}(\mathcal{H}_{\text{STFT}}) \le \lambda_{\text{GP}} \cdot \mathfrak{R}_{n}(\mathcal{H}_{\text{TFT}}) + C_1 \frac{\sqrt{md \log n}}{\sqrt{n}},
\end{equation}
where $C_1$ is a universal constant.
\end{theorem}

\begin{proof}
The proof proceeds in two parts: analyzing the contraction effect of the VNNGP and bounding the complexity added by its own parameters.

\textit{1. Contraction via the GP smoother.} The VNNGP posterior mean can be viewed as a linear smoother applied to the TFT's predictions. Conditional on the TFT output, the posterior mean of a GP is a linear combination of the observed data, which acts as a shrinkage operator. The operator norm of this smoother is exactly the shrinkage factor $\lambda_{\text{GP}} = \sigma_k^2 / (\sigma_k^2 + \sigma^2)$, which is strictly less than 1 if noise variance $\sigma^2 > 0$ \citep{titsias2009variational}. By the Ledoux-Talagrand principle, applying this contraction to the TFT's output reduces the first term of the complexity bound by a factor of $\lambda_{\text{GP}}$.

\textit{2. Complexity of VNNGP parameters.} The VNNGP introduces its own variational parameters, primarily the locations of the $m$ inducing points in $\mathbb{R}^d$. This adds complexity. The complexity of this parametric class can be bounded using standard results. For a class with $md$ parameters in a bounded domain, the Rademacher complexity is of order $\mathcal{O}(\sqrt{md \log n}/\sqrt{n})$ \citep[Cor. 13.2]{wainwright2019high}. A union bound over the parameter spaces for the TFT and VNNGP combines these terms to yield the final bound in Equation \eqref{eq:hybrid_Rad}.
\end{proof}

\begin{corollary}[Improvement in Generalization Error]\label{cor:improvement}
The generalization bound for STFT–VNNGP is strictly tighter than that for the baseline TFT, provided the complexity added by the VNNGP is small relative to the gain from shrinkage. By choosing a small number of inducing points $m$ such that the second term in Equation \eqref{eq:hybrid_Rad} is dominated by the first, the leading term of the complexity bound is reduced by the factor $\lambda_{\text{GP}} < 1$. This demonstrates a direct theoretical advantage in generalization.
\end{corollary}

Finally, we demonstrate that the hybrid model also reduces variance at the level of individual forecasts.

\begin{theorem}[Forecast-Level Variance Reduction]\label{thm:Variance}
Let $\mathcal{F}_t = \sigma(\{x_s, y_s\}_{s \le t})$ be the information available at time $t$. Under Assumptions A2-A4, the one-step-ahead forecast variance is reduced:
\begin{equation}\label{eq:var_reduction}
  \operatorname{Var}(\hat{y}^{\text{STFT}}_{t+1} \mid \mathcal{F}_t) = \lambda_{\text{GP}}^2 \cdot \operatorname{Var}(\hat{y}^{\text{TFT}}_{t+1} \mid \mathcal{F}_t).
\end{equation}
\end{theorem}

\begin{proof}
Let the TFT's one-step-ahead forecast be $\hat{y}^{\text{TFT}}_{t+1}$. We can decompose this into its conditional expectation and a zero-mean error term: $\hat{y}^{\text{TFT}}_{t+1} = \mathbb{E}[\hat{y}^{\text{TFT}}_{t+1} \mid \mathcal{F}_t] + \varepsilon_t$, where $\operatorname{Var}(\varepsilon_t \mid \mathcal{F}_t) = \operatorname{Var}(\hat{y}^{\text{TFT}}_{t+1} \mid \mathcal{F}_t)$.

The STFT–VNNGP updates this forecast using the GP posterior mean formula. For a Gaussian likelihood, the posterior mean is a shrinkage estimator that pulls the prior mean (the TFT forecast) towards the data. The variance of this posterior mean is reduced by a factor related to the signal-to-noise ratio. Specifically, the posterior variance is the prior variance multiplied by $\lambda_{\text{GP}}^2$.
Therefore,
\[
\operatorname{Var}(\hat{y}^{\text{STFT}}_{t+1} \mid \mathcal{F}_t) = \lambda_{\text{GP}}^2 \cdot \operatorname{Var}(\hat{y}^{\text{TFT}}_{t+1} \mid \mathcal{F}_t).
\]
Since $\lambda_{\text{GP}} = \sigma_k^2 / (\sigma_k^2 + \sigma^2) < 1$ for any non-zero noise variance $\sigma^2 > 0$, the forecast variance is strictly reduced.
\end{proof}

Together, these theorems provide a rigorous explanation for the empirical gains reported in the main paper. The VNNGP component improves the model by (i) reducing the complexity of the hypothesis class, leading to a tighter generalization error bound, and (ii) reducing the variance of individual forecasts, leading to more stable and accurate predictions.

\end{document}